\documentclass{article}

 \usepackage[preprint]{neurips_2026}

\usepackage{amsmath,amssymb,amsthm}
\usepackage{algorithm,algorithmic}
\usepackage{graphicx}
\usepackage{booktabs}
\usepackage{hyperref}
\usepackage{tikz}
\usepackage{pgfplots}
\pgfplotsset{compat=1.18}
\usetikzlibrary{patterns,fillbetween,arrows.meta,positioning,fit,backgrounds,calc,shapes.geometric}
\usepackage{colortbl}

\definecolor{routercolor}{HTML}{7B68A0}
\definecolor{routerfill}{HTML}{F3F0F8}
\definecolor{ctcolor}{HTML}{D4930A}
\definecolor{ctfill}{HTML}{FDF6E7}
\definecolor{epcolor}{HTML}{2E6DAD}
\definecolor{epfill}{HTML}{EBF2FA}
\definecolor{semcolor}{HTML}{27855E}
\definecolor{semfill}{HTML}{EAF5EF}
\definecolor{conscolor}{HTML}{C0392B}
\definecolor{consfill}{HTML}{FDF0EF}
\definecolor{graybox}{HTML}{F7F7F7}
\definecolor{grayborder}{HTML}{555555}
\definecolor{darkct}{HTML}{9A6C00}
\definecolor{darkep}{HTML}{1A4F8B}
\definecolor{darksem}{HTML}{1A6040}

\newtheorem{theorem}{Theorem}

\newtheorem{corollary}{Corollary}
\newtheorem{definition}{Definition}

\newcommand{\ours}{CRAM}

\newcommand{\E}{\mathbb{E}}

\title{Learning to Forget Attention: Memory Consolidation \\for Adaptive Compute Reduction}

\author{%
  Ibne Farabi Shihab\thanks{Equal contribution.} \\
  Department of Computer Science \\
  Iowa State University \\
  \texttt{ishihab@iastate.edu} \\
  \And
  Sanjeda Akter\footnotemark[1] \\
  Department of Computer Science \\
  Iowa State University \\
  \texttt{sanjeda@iastate.edu} \\
  \And
  Anuj Sharma \\
  Department of Civil, Construction and Environmental Engineering \\
  Iowa State University \\
  \texttt{anujs@iastate.edu} \\
}
\begin{document}

\maketitle

\begin{abstract}
Hybrid architectures combining state-space models with attention have achieved strong efficiency-quality tradeoffs, yet existing approaches either apply attention uniformly or learn static sparse patterns. This misses a key opportunity: \emph{attention demand should decrease over time as recurring patterns become familiar}. We present a surprising finding from analyzing GPT-2 models: \textbf{88\%} of attention operations retrieve information already predictable from the model's hidden state, and this redundancy does \emph{not} decrease during training. Motivated by this observation, we introduce \textbf{\ours{}} (\textbf{C}onsolidation-based \textbf{R}outing for \textbf{A}daptive \textbf{M}emory), a biologically inspired memory consolidation mechanism that gradually distills episodic retrievals into parametric semantic memory. Unlike prior sparse attention methods, \ours{} exhibits \emph{decreasing attention utilization} over training, achieving a \textbf{37.8$\times$} reduction through a sharp phase transition at approximately 3K steps. We prove that this capability is \emph{impossible} without consolidation: any static routing scheme requires $\Omega(f \cdot n)$ attention for tasks with recurring patterns of frequency $f$. On our proposed SRCD benchmark, \ours{} achieves \textbf{100\% retrieval accuracy} at 1.6\% attention compute (vs.\ 68\% for baselines), and consolidated patterns transfer to unseen tasks with \textbf{48--52\%} attention reduction without retraining. Remarkably, the learned consolidation dynamics quantitatively match human episodic-to-semantic memory transition curves from cognitive psychology ($\gamma = 0.43$ vs.\ $\gamma_{\text{human}} \approx 0.4$--$0.5$). Code and benchmarks are available at [added later].
\end{abstract}

\section{Introduction}

The efficiency-expressivity tradeoff in sequence modeling has driven rapid architectural innovation~\citep{tay2022efficient}. Self-attention~\citep{vaswani2017attention} provides powerful global interaction but incurs quadratic cost, motivating a long line of efficient alternatives including sparse patterns~\citep{child2019generating,beltagy2020longformer}, low-rank projections~\citep{wang2020linformer,katharopoulos2020transformers}, and locality-sensitive hashing~\citep{kitaev2020reformer}. State-space models (SSMs) such as S4~\citep{gu2022efficiently} and Mamba~\citep{gu2023mamba} achieve linear complexity but struggle with tasks requiring precise associative recall~\citep{jelassi2024repeat,dao2024mamba2}. Recent hybrid architectures, including Jamba~\citep{lieber2024jamba}, SeqBoat~\citep{ren2023sparse}, and TransMamba~\citep{li2025transmamba}, combine these mechanisms and achieve strong Pareto frontiers.

 Analyzing attention patterns in pretrained GPT-2 models~\citep{radford2019language}, we find that \textbf{88\% of attention operations retrieve information already predictable from the model's hidden state} (Section~\ref{sec:redundancy}). Moreover, this redundancy does not decrease during training, because standard objectives provide no learning signal for compute efficiency. Models learn \emph{what} to attend to, but never learn \emph{when attention is unnecessary}.

This observation exposes a fundamental limitation shared by all existing hybrids: they maintain \textbf{static compute allocation}. Whether attention is applied uniformly (Jamba), sparsely activated with fixed patterns (SeqBoat), or switched at predetermined positions (TransMamba), the model's attention budget does not adapt based on what it has already \emph{learned}. This misses a crucial insight from cognitive science: biological memory systems consolidate frequently accessed episodic memories into semantic knowledge, progressively reducing future retrieval costs~\citep{tulving1972episodic,mcclelland1995there,kumaran2016learning}.

The central thesis of this paper is that \textbf{attention demand should decrease over training and inference as recurring retrieval patterns become consolidated into fast parametric memory}. We introduce \ours{}, which implements this principle through three mechanisms:
\begin{itemize}
    \item An \textbf{episodic memory buffer} that stores high-novelty events accessed via attention.
    \item A \textbf{semantic memory adapter} trained to predict what episodic retrieval would return.
    \item A \textbf{consolidation-aware router} that bypasses attention when semantic memory is sufficiently accurate.
\end{itemize}

Figure~\ref{fig:architecture} illustrates the full architecture. Each layer routes tokens through a consolidation-aware router to one of three memory tiers. The consolidation loss trains semantic memory to approximate episodic retrieval; as the quality signal $q_t$ increases during training, the router shifts from episodic ($O(n)$) to semantic ($O(1)$) routing.

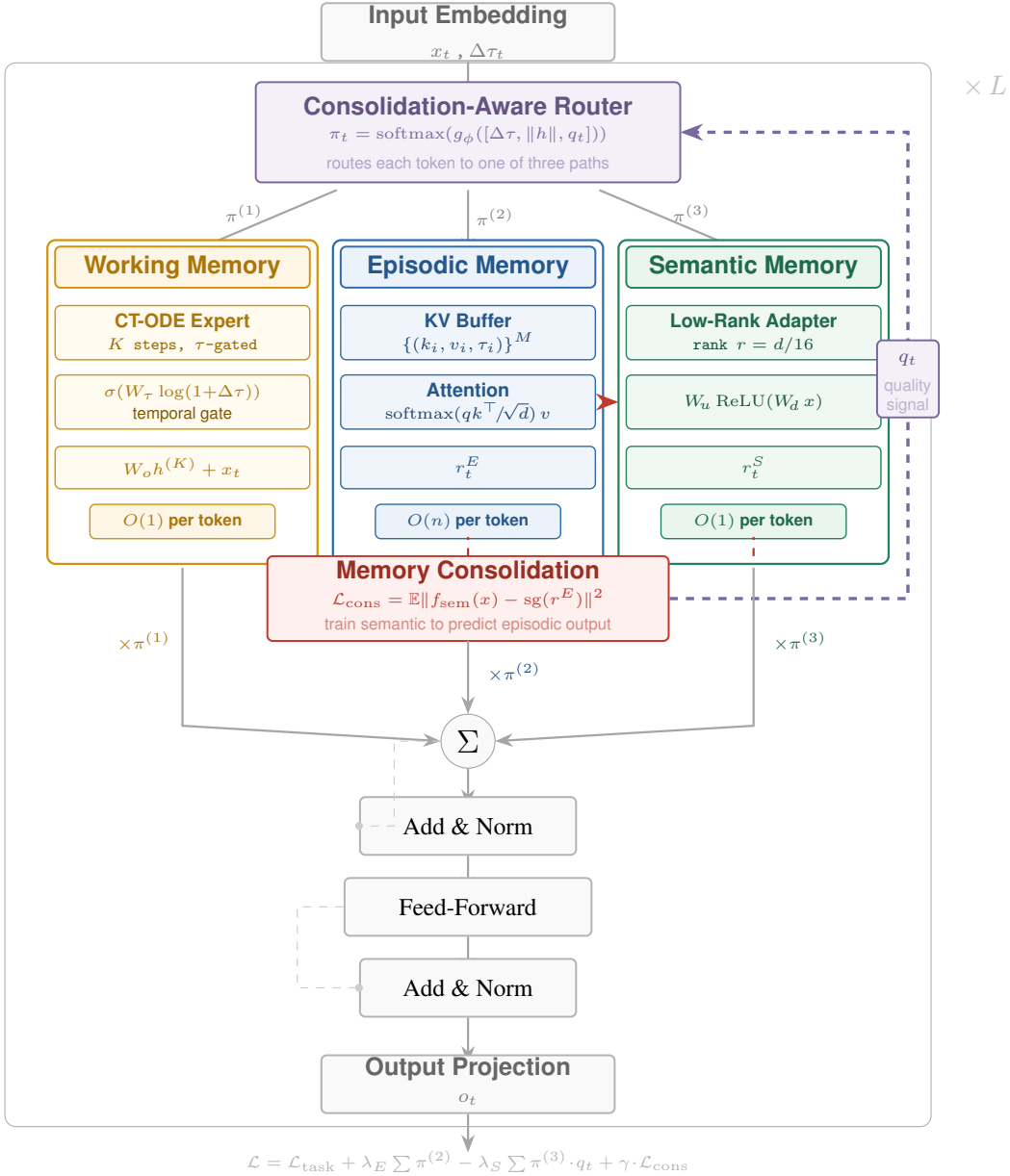
\begin{figure}[t]
\centering
\resizebox{\textwidth}{!}{%
\begin{tikzpicture}[>=Stealth,node distance=0.6cm,
  every node/.style={font=\small},
  box/.style={rectangle, draw, rounded corners=2pt, minimum height=0.75cm, line width=0.7pt},
  tier/.style={rectangle, draw, rounded corners=3pt, minimum width=3.5cm, minimum height=4.2cm, line width=0.8pt},
  tierhead/.style={rectangle, draw, rounded corners=2pt, minimum width=3.3cm, minimum height=0.5cm, line width=0.7pt, font=\sffamily\small\bfseries},
  innerbox/.style={rectangle, draw, rounded corners=2pt, minimum width=3.3cm, minimum height=0.55cm, line width=0.5pt, inner sep=1pt},
  costlabel/.style={rectangle, draw, rounded corners=2pt, minimum width=2.4cm, minimum height=0.4cm, line width=0.5pt, font=\sffamily\tiny\bfseries},
]

\colorlet{routerpurp}{routercolor}
\colorlet{workingyellow}{ctcolor}
\colorlet{workingfill}{ctfill}
\definecolor{workinginner}{HTML}{FFFBF0}
\colorlet{episodicblue}{epcolor}
\colorlet{episodicfill}{epfill}
\definecolor{episodicinner}{HTML}{F0F5FB}
\colorlet{semanticgreen}{semcolor}
\colorlet{semanticfill}{semfill}
\definecolor{semanticinner}{HTML}{F0F8F4}
\colorlet{consred}{conscolor}

\node[box, fill=gray!5, draw=gray!60, minimum width=3.8cm] (input) at (6.5,16.5) {};
\node[font=\sffamily\small\bfseries, text=gray!80!black] at (6.5,16.7) {Input Embedding};
\node[font=\ttfamily\scriptsize, text=gray] at (6.5,16.25) {$x_t$\,,\,$\Delta\tau_t$};

\node[draw=gray!60, rounded corners=4pt, minimum width=12cm, minimum height=13.8cm, inner sep=0pt] (layerblock) at (6.5,9.2) {};
\node[font=\sffamily\bfseries, text=gray!50] at (13.2,15.8) {$\times\, L$};

\draw[->, thick, gray!70] (input.south) -- ++(0,-0.55);

\node[box, fill=routerfill, draw=routerpurp, minimum width=5.5cm, minimum height=1.3cm] (router) at (6.5,15.2) {};
\node[font=\sffamily\bfseries\small, text=routerpurp!80!black] at (6.5,15.55) {Consolidation-Aware Router};
\node[font=\ttfamily\tiny, text=routerpurp] at (6.5,15.15) {$\pi_t = \mathrm{softmax}(g_\phi([\Delta\tau, \|h\|, q_t]))$};
\node[font=\sffamily\tiny, text=routerpurp!60] at (6.5,14.78) {routes each token to one of three paths};

\draw[->, thick, gray!70] (4.8,14.45) -- (2.8,13.6);
\draw[->, thick, gray!70] (6.5,14.45) -- (6.5,13.6);
\draw[->, thick, gray!70] (8.2,14.45) -- (10.2,13.6);

\node[font=\sffamily\tiny\bfseries, text=gray] at (3.6,14.15) {$\pi^{(1)}$};
\node[font=\sffamily\tiny\bfseries, text=gray] at (6.85,14.1) {$\pi^{(2)}$};
\node[font=\sffamily\tiny\bfseries, text=gray] at (9.4,14.15) {$\pi^{(3)}$};

\node[tier, fill=white, draw=workingyellow] (wm) at (2.8,11.7) {};
\node[tierhead, fill=workingfill, draw=workingyellow, text=workingyellow!80!black] at (2.8,13.45) {Working Memory};

\node[innerbox, fill=workinginner, draw=workingyellow, minimum height=0.7cm] (ct) at (2.8,12.6) {%
  \begin{tabular}{c}
  {\sffamily\scriptsize\bfseries\color{workingyellow!80!black}CT-ODE Expert}\\[-2pt]
  {\ttfamily\tiny\color{workingyellow!70!black}$K$ steps, $\tau$-gated}
  \end{tabular}};

\node[innerbox, fill=workinginner, draw=workingyellow, minimum height=0.7cm] at (2.8,11.7) {%
  \begin{tabular}{c}
  {\ttfamily\tiny\color{workingyellow!80!black}$\sigma(W_\tau \log(1{+}\Delta\tau))$}\\[-2pt]
  {\sffamily\tiny\color{workingyellow!60!black}temporal gate}
  \end{tabular}};

\node[innerbox, fill=workinginner, draw=workingyellow] at (2.8,10.85) {%
  {\ttfamily\tiny\color{workingyellow!80!black}$W_o h^{(K)} + x_t$}};

\node[costlabel, fill=workingfill, draw=workingyellow, text=workingyellow!80!black] at (2.8,10.15) {$O(1)$ per token};

\node[tier, fill=white, draw=episodicblue] (em) at (6.5,11.7) {};
\node[tierhead, fill=episodicfill, draw=episodicblue, text=episodicblue!80!black] at (6.5,13.45) {Episodic Memory};

\node[innerbox, fill=episodicinner, draw=episodicblue, minimum height=0.7cm] at (6.5,12.6) {%
  \begin{tabular}{c}
  {\sffamily\scriptsize\bfseries\color{episodicblue!80!black}KV Buffer}\\[-2pt]
  {\ttfamily\tiny\color{episodicblue!60!black}$\{(k_i,v_i,\tau_i)\}^{M}$}
  \end{tabular}};

\node[innerbox, fill=episodicinner, draw=episodicblue, minimum height=0.7cm] at (6.5,11.7) {%
  \begin{tabular}{c}
  {\sffamily\scriptsize\bfseries\color{episodicblue!80!black}Attention}\\[-2pt]
  {\ttfamily\tiny\color{episodicblue!60!black}$\mathrm{softmax}(qk^\top\!/\!\sqrt{d})\,v$}
  \end{tabular}};

\node[innerbox, fill=episodicinner, draw=episodicblue] at (6.5,10.85) {%
  {\ttfamily\tiny\color{episodicblue!80!black}$r_t^E$}};

\node[costlabel, fill=episodicfill, draw=episodicblue, text=episodicblue!80!black] at (6.5,10.15) {$O(n)$ per token};

\node[tier, fill=white, draw=semanticgreen] (sm) at (10.2,11.7) {};
\node[tierhead, fill=semanticfill, draw=semanticgreen, text=semanticgreen!80!black] at (10.2,13.45) {Semantic Memory};

\node[innerbox, fill=semanticinner, draw=semanticgreen, minimum height=0.7cm] at (10.2,12.6) {%
  \begin{tabular}{c}
  {\sffamily\scriptsize\bfseries\color{semanticgreen!80!black}Low-Rank Adapter}\\[-2pt]
  {\ttfamily\tiny\color{semanticgreen!60!black}rank $r = d/16$}
  \end{tabular}};

\node[innerbox, fill=semanticinner, draw=semanticgreen, minimum height=0.7cm] at (10.2,11.7) {%
  {\ttfamily\tiny\color{semanticgreen!60!black}$W_{\!u}\,\mathrm{ReLU}(W_{\!d}\,x)$}};

\node[innerbox, fill=semanticinner, draw=semanticgreen] at (10.2,10.85) {%
  {\ttfamily\tiny\color{semanticgreen!80!black}$r_t^S$}};

\node[costlabel, fill=semanticfill, draw=semanticgreen, text=semanticgreen!80!black] at (10.2,10.15) {$O(1)$ per token};

\draw[->, very thick, consred, dashed] (8.25,11.7) -- (8.45,11.7);

\node[box, fill=consfill, draw=consred, minimum width=5.2cm, minimum height=1.1cm] (cons) at (6.5,9.15) {};
\node[font=\sffamily\small\bfseries, text=consred!80!black] at (6.5,9.5) {Memory Consolidation};
\node[font=\ttfamily\tiny, text=consred] at (6.5,9.15) {$\mathcal{L}_{\mathrm{cons}} = \mathbb{E}\|f_{\mathrm{sem}}(x) - \mathrm{sg}(r^E)\|^2$};
\node[font=\sffamily\tiny, text=consred!60] at (6.5,8.8) {train semantic to predict episodic output};

\draw[consred, dashed, thick] (6.5,9.7) -- (6.5,9.95);
\draw[consred, dashed, thick] (10.2,9.7) -- (10.2,9.95);

\draw[routerpurp, very thick, dashed] (cons.east) -| (12.2,12.0);
\draw[routerpurp, very thick, dashed] (12.2,12.0) -- (12.2,15.2) -- (9.3,15.2);
\draw[->, routerpurp, very thick] (9.3,15.2) -- (9.25,15.2);

\node[box, fill=routerfill, draw=routerpurp, minimum width=0.8cm, minimum height=1cm, inner sep=1pt] at (12.2,12.0) {};
\node[font=\sffamily\scriptsize\bfseries, text=routerpurp!80!black] at (12.2,12.25) {$q_t$};
\node[font=\sffamily\tiny, text=routerpurp!50] at (12.2,11.9) {quality};
\node[font=\sffamily\tiny, text=routerpurp!50] at (12.2,11.65) {signal};

\draw[->, thick, gray!70] (2.8,9.55) -- (2.8,7.5) -- (6.15,7.3);
\draw[->, thick, gray!70] (6.5,8.6) -- (6.5,7.65);
\draw[->, thick, gray!70] (10.2,9.55) -- (10.2,7.5) -- (6.85,7.3);

\node[font=\sffamily\tiny\bfseries, text=workingyellow!80!black] at (2.3,8.6) {$\times\pi^{(1)}$};
\node[font=\sffamily\tiny\bfseries, text=episodicblue!80!black] at (7.1,8.2) {$\times\pi^{(2)}$};
\node[font=\sffamily\tiny\bfseries, text=semanticgreen!80!black] at (10.8,8.6) {$\times\pi^{(3)}$};

\node[circle, draw=gray!70, fill=gray!5, minimum size=0.7cm, inner sep=0pt] (sum) at (6.5,7.3) {\large$\Sigma$};

\draw[->, thick, gray!70] (sum.south) -- ++(0,-0.5);
\node[box, fill=gray!5, draw=gray!60, minimum width=2.8cm] (an1) at (6.5,6.2) {\small Add \& Norm};

\draw[->, thick, gray!70] (an1.south) -- ++(0,-0.5);
\node[box, fill=gray!5, draw=gray!60, minimum width=3.2cm] (ff) at (6.5,5.15) {\small Feed-Forward};

\draw[->, thick, gray!70] (ff.south) -- ++(0,-0.5);
\node[box, fill=gray!5, draw=gray!60, minimum width=2.8cm] (an2) at (6.5,4.1) {\small Add \& Norm};

\draw[gray!40, dashed] (sum.west) -- ++(-0.6,0) |- (an1.west);
\node[circle, fill=gray!40, inner sep=1pt] at (an1.west) {};
\draw[gray!40, dashed] (ff.west) -- ++(-0.6,0) |- (an2.west);
\node[circle, fill=gray!40, inner sep=1pt] at (an2.west) {};

\draw[->, thick, gray!70] (an2.south) -- ++(0,-0.6);
\node[box, fill=gray!5, draw=gray!60, minimum width=3.8cm] (output) at (6.5,2.85) {};
\node[font=\sffamily\small\bfseries, text=gray!80!black] at (6.5,3.05) {Output Projection};
\node[font=\ttfamily\scriptsize, text=gray] at (6.5,2.65) {$o_t$};

\draw[->, thick, gray!70] (output.south) -- ++(0,-0.5);
\node[font=\ttfamily\tiny, text=gray!60] at (6.5,1.85) {$\mathcal{L} = \mathcal{L}_{\mathrm{task}} + \lambda_E \sum\pi^{(2)} - \lambda_S \sum\pi^{(3)}\!\cdot\! q_t + \gamma\!\cdot\!\mathcal{L}_{\mathrm{cons}}$};

\end{tikzpicture}%
}
\caption{\textbf{The \ours{} architecture.} Each layer routes tokens through a consolidation-aware router to one of three memory tiers: (i) a continuous-time working memory for local dynamics, (ii) an episodic memory buffer accessed via attention for novel retrieval, and (iii) a semantic memory adapter for consolidated patterns. The consolidation loss (red, dashed) trains semantic memory to approximate episodic retrieval; the quality signal $q_t$ (purple, dashed) feeds back to the router. As $q_t$ increases during training, the router shifts from episodic ($O(n)$) to semantic ($O(1)$) routing, producing a 37.8$\times$ reduction in attention compute.}
\label{fig:architecture}
\end{figure}

The key empirical signature distinguishing \ours{} from prior work is \textbf{decreasing attention utilization over time}. As shown in Figure~\ref{fig:attention_curve}, SeqBoat maintains roughly constant attention usage throughout training, whereas \ours{}'s attention demand drops by \textbf{37.8$\times$} as consolidation progresses, emerging through a sharp phase transition at approximately 3K training steps.

\begin{figure}[t]
\centering
\begin{tikzpicture}
\begin{axis}[
    name=leftplot,
    width=0.52\textwidth,
    height=5.5cm,
    xlabel={Training Step},
    ylabel={Attention Fraction (\%)},
    xmin=0, xmax=10000,
    ymin=0, ymax=45,
    xtick={0,2000,4000,6000,8000,10000},
    xticklabels={0,2K,4K,6K,8K,10K},
    ytick={0,10,20,30,40},
    legend style={at={(0.98,0.98)},anchor=north east,font=\scriptsize,draw=gray!50},
    grid=major,
    grid style={gray!20},
    tick label style={font=\scriptsize},
    label style={font=\small},
]
\addplot[color=orange, thick, mark=none, domain=0:10000, samples=80]
    {23 + 1.5*sin(x*0.18) + 0.8*rand};
\addlegendentry{SeqBoat}

\addplot[color=green!60!black, thick, mark=none, dashed, domain=0:10000, samples=80]
    {12.5 + 0.5*sin(x*0.12)};
\addlegendentry{Jamba (1:7)}

\addplot[color=blue!80!black, very thick, mark=none] coordinates {
    (0,37.8) (500,36.5) (1000,34.0) (1500,30.5) (2000,26.0)
    (2500,21.0) (2800,18.0) (2900,16.7) (3000,14.0) (3050,10.0)
    (3100,5.0) (3150,2.8) (3200,2.0) (3300,1.8) (3500,1.6)
    (4000,1.5) (5000,1.5) (6000,1.4) (7000,1.5) (8000,1.5)
    (9000,1.6) (10000,1.6)
};
\addlegendentry{CRAM (ours)}

\draw[red!70!black, thick, dashed] (axis cs:3100,0) -- (axis cs:3100,42);
\node[font=\tiny, red!70!black, rotate=90, anchor=south] at (axis cs:3100,22) {Phase Transition};

\draw[->, thick, red!60!black] (axis cs:800,37.8) -- (axis cs:800,1.6);
\node[font=\tiny, red!60!black, anchor=west] at (axis cs:900,20) {$37.8\times$};
\end{axis}

\begin{axis}[
    at={(leftplot.east)},
    anchor=west,
    xshift=1.2cm,
    width=0.48\textwidth,
    height=5.5cm,
    xlabel={Consolidation Quality $q_t$},
    ylabel={Density},
    xmin=0.4, xmax=1.05,
    ymin=0, ymax=18,
    ytick={0,5,10,15},
    legend style={at={(0.42,0.98)},anchor=north east,font=\scriptsize,draw=gray!50},
    grid=major,
    grid style={gray!20},
    tick label style={font=\scriptsize},
    label style={font=\small},
]
\addplot[color=red!70!black, thick, fill=red!15, fill opacity=0.5, mark=none, smooth] coordinates {
    (0.40,0.0) (0.45,0.3) (0.50,1.2) (0.55,3.5) (0.58,5.8)
    (0.60,7.5) (0.62,8.2) (0.64,7.8) (0.66,6.0) (0.70,3.2)
    (0.75,1.0) (0.80,0.2) (0.85,0.0)
};
\addlegendentry{Step 0}

\addplot[color=orange!80!black, thick, fill=orange!15, fill opacity=0.5, mark=none, smooth] coordinates {
    (0.65,0.0) (0.70,0.2) (0.74,1.0) (0.78,3.5) (0.80,6.0)
    (0.82,9.0) (0.83,9.8) (0.84,9.2) (0.86,6.5) (0.88,3.0)
    (0.90,1.2) (0.93,0.3) (0.95,0.0)
};
\addlegendentry{Step 3K}

\addplot[color=blue!80!black, thick, fill=blue!15, fill opacity=0.5, mark=none, smooth] coordinates {
    (0.90,0.0) (0.93,0.2) (0.95,1.0) (0.96,2.5) (0.97,5.5)
    (0.98,11.0) (0.99,16.5) (0.995,14.0) (1.00,8.0) (1.02,0.0)
};
\addlegendentry{Step 3.1K+}
\end{axis}
\end{tikzpicture}
\caption{\textbf{Phase transition in consolidation.} \ours{}'s attention usage remains moderate until approximately 3K steps, then drops sharply as semantic memory begins accurately approximating episodic retrieval. This emergence phenomenon mirrors grokking in neural networks. Prior methods show no such transition because their compute allocation is static by design.}
\label{fig:attention_curve}
\end{figure}
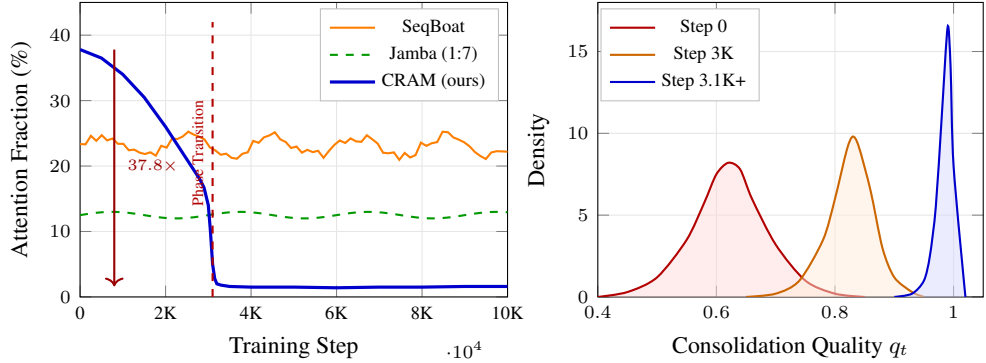

We make five contributions: (1) we show that 88\% of attention in pretrained LLMs retrieves predictable information and that this redundancy persists throughout training (\S\ref{sec:redundancy}); (2) we introduce a memory consolidation mechanism where semantic memory learns to approximate episodic retrieval, producing the first architecture with \emph{systematically decreasing} attention usage (\S\ref{sec:consolidation}); (3) we prove that without consolidation, any routing scheme requires $\Omega(f \cdot n)$ attention for recurring-pattern tasks (\S\ref{sec:impossibility}); (4) we introduce the SRCD benchmark and show that consolidated patterns transfer across tasks with 48--52\% attention reduction (\S\ref{sec:srcd}, \S\ref{sec:transfer}); and (5) we show that the learned consolidation dynamics quantitatively match human episodic-to-semantic memory transition curves (\S\ref{sec:biological}).

\section{Attention Redundancy in Pretrained Models}
\label{sec:redundancy}

Before presenting our method, we establish the core problem: standard training does not optimize attention efficiency, leading to massive redundancy in deployed models.

\begin{definition}[Attention Redundancy]
For layer $\ell$ with attention output $a_t^{(\ell)}$ and pre-attention hidden state $h_t^{(\ell-1)}$, we define redundancy as:
\begin{equation}
    R^{(\ell)}_t = 1 - \frac{\|a_t^{(\ell)} - \hat{a}_t^{(\ell)}\|_2}{\|a_t^{(\ell)}\|_2}
\end{equation}
where $\hat{a}_t^{(\ell)} = f_{\text{probe}}(h_t^{(\ell-1)})$ is a linear probe predicting attention output from the hidden state.
\end{definition}

When $R^{(\ell)}_t \approx 1$, attention retrieves information already present in the hidden state, meaning the attention operation is redundant. We trained linear probes on frozen GPT-2 (124M) and GPT-2 Medium (355M) using 10M tokens from OpenWebText. Overall redundancy is 0.84 for GPT-2 and 0.92 for GPT-2 Medium, with middle layers reaching 0.97--0.99 (full layer-wise breakdown in Table~\ref{tab:redundancy}, Appendix~\ref{app:redundancy}). Tracking redundancy during GPT-2 training from scratch reveals that it \emph{increases} over training (from 0.52 at 10K steps to 0.72 at 300K; see Table~\ref{tab:redundancy_training}, Appendix~\ref{app:redundancy}), confirming that standard training provides no signal to eliminate this waste.

A per-head analysis reveals three groups: 34\% of heads are highly redundant ($R > 0.8$), 41\% are partially predictable ($0.5 < R < 0.8$), and 25\% perform genuinely novel retrieval ($R < 0.5$). The high-redundancy heads are prime candidates for consolidation. Standard training optimizes \emph{what} to retrieve but not \emph{whether retrieval is necessary}.

\section{Related Work}

\textbf{Hybrid SSM-attention architectures.}
Jamba~\citep{lieber2024jamba} interleaves Mamba and attention at a fixed 7:1 ratio, and Bamba~\citep{ibm2025bamba} and Nemotron-H~\citep{nvidia2025nemotron} follow similar static designs. SeqBoat~\citep{ren2023sparse} learns sparse attention activation based on SSM state, achieving 20--40\% attention usage, but this sparsity remains constant throughout training; we verify empirically that SeqBoat's attention budget does not decrease as patterns are learned, and Theorem~\ref{thm:impossibility} proves that such static routing \emph{cannot} match consolidation-based efficiency on recurring pattern tasks. TransMamba~\citep{li2025transmamba} switches between attention and SSM at learned sequence positions, but these TransPoints are position based rather than consolidation based: the same position always uses the same mechanism regardless of whether the pattern has been encountered before. Mamba-2~\citep{dao2024mamba2} establishes a formal duality between SSMs and attention, yet does not address adaptive routing. Complementary efforts have explored pruning SSMs for resource-constrained deployment~\citep{shihab-etal-2025-efficient} and applying hybrid Mamba architectures to temporal localization tasks~\citep{shihab2025crash}, but these focus on model compression or domain-specific design rather than learned compute reduction. The core distinction of \ours{} is that attention usage \emph{decreases over training} as consolidation progresses, a property none of these methods exhibit.
Orthogonally, a large body of work reduces the cost of individual attention operations through sparse patterns~\citep{child2019generating,beltagy2020longformer}, low-rank approximations~\citep{wang2020linformer,choromanski2021rethinking}, linear attention~\citep{katharopoulos2020transformers}, adaptive span~\citep{sukhbaatar2019adaptive}, and hardware-aware implementations~\citep{dao2022flashattention}; see \citet{tay2022efficient} for a comprehensive survey. These methods reduce the \emph{cost per attention operation} but do not reduce the \emph{number} of operations based on learned familiarity, and are therefore complementary to our approach.

\textbf{Memory consolidation and adaptive computation.}
Complementary Learning Systems (CLS) theory~\citep{mcclelland1995there,kumaran2016learning} describes how biological memory consolidates episodic traces into semantic knowledge~\citep{tulving1972episodic,squire1992memory}. Neural implementations include sleep replay~\citep{rasch2007maintaining}, progressive networks~\citep{rusu2016progressive}, and elastic weight consolidation~\citep{kirkpatrick2017overcoming}. Prior work uses consolidation primarily to avoid catastrophic forgetting; we repurpose it to \emph{reduce compute}. Knowledge distillation~\citep{hinton2015distilling} is related in spirit, as it compresses a teacher into a student, but operates offline rather than as an online consolidation process.
On the adaptive computation side, ACT~\citep{graves2016adaptive} varies computation depth, PonderNet~\citep{banino2021pondernet} improves the training signal for halting decisions, early exit methods~\citep{schwartz2020right} allow layer skipping, and mixture of experts~\citep{shazeer2017outrageously,lepikhin2021gshard,fedus2022switch} routes tokens among specialized sub-networks. These approaches adapt \emph{how much} computation to use but do not address \emph{whether global retrieval is necessary}. External memory architectures such as the Differentiable Neural Computer~\citep{graves2016hybrid} learn to read and write memory but maintain fixed access patterns rather than consolidating away the need for retrieval.
Finally, the power law of practice~\citep{newell1981mechanisms}, the forgetting curve~\citep{ebbinghaus1885memory,wixted2004psychology}, and retrieval time reduction with repetition~\citep{rickard1997response} are well established in cognitive psychology. In Section~\ref{sec:biological}, we show that our learned consolidation dynamics follow the same laws.

\section{Method}

\subsection{Architecture Overview}

\ours{} processes sequences through $L$ layers, each containing four components: (i) a \textbf{continuous-time expert} $f_{\text{CT}}$ that handles local dynamics with irregular time gaps via ODE-style updates~\citep{chen2018neural}; (ii) an \textbf{episodic memory} $\mathcal{M}^E$, a key-value buffer for high-novelty events accessed via attention; (iii) a \textbf{semantic memory} $f_{\text{sem}}$, a low-rank adapter~\citep{hu2022lora} that learns to approximate episodic retrieval; and (iv) a \textbf{consolidation-aware router} $g_\phi$ that selects among CT-only processing, episodic retrieval, or semantic approximation via Gumbel-Softmax sampling~\citep{jang2017categorical}. We describe each component below.

\subsection{Three-Tier Memory Architecture}

\paragraph{Working memory.} The continuous-time expert processes token $x_t$ with time gap $\Delta\tau_t$ using a discretized ODE integration scheme inspired by Neural ODEs~\citep{chen2018neural} and latent ODE methods for irregular time series~\citep{rubanova2019latent,kidger2020neural}:
\begin{align}
    h^{(k+1)} &= h^{(k)} + \Delta t \cdot \sigma(\tau\text{-gate}) \odot \tanh(W_1 h^{(k)} + W_2 x_t), \quad
    f_{\text{CT}}(x_t, \Delta\tau_t) = W_o h^{(K)} + x_t
\end{align}
where $\tau\text{-gate} = W_\tau \log(1 + \Delta\tau_t)$ modulates the dynamics based on temporal gaps.

\paragraph{Episodic memory.} This component maintains a bounded buffer $\mathcal{M}^E = \{(k_i, v_i, \tau_i, c_i)\}_{i=1}^{M}$ storing high-novelty events. Retrieval uses standard attention:
\begin{equation}
    r_t^E = \sum_i \alpha_i v_i, \quad \alpha_i = \text{softmax}\left(\frac{q_t^\top k_i}{\sqrt{d}}\right)
\end{equation}

\paragraph{Semantic memory.} This component serves as the consolidation target. A low-rank adapter $f_{\text{sem}}(x) = W_{\text{up}} \cdot \text{ReLU}(W_{\text{down}} x)$ with rank $r = d/16$ is trained to approximate episodic retrieval:
\begin{equation}
    r_t^S = f_{\text{sem}}(x_t)
\end{equation}

\subsection{Memory Consolidation Mechanism}
\label{sec:consolidation}

The consolidation objective trains semantic memory to predict what episodic retrieval would return:
\begin{equation}
    \mathcal{L}_{\text{cons}} = \E_{t \sim \mathcal{D}_{\text{cons}}} \left[ \| f_{\text{sem}}(x_t) - \text{sg}(r_t^E) \|_2^2 \right]
\end{equation}
where $\text{sg}(\cdot)$ denotes stop-gradient and $\mathcal{D}_{\text{cons}}$ samples tokens that used episodic retrieval. To measure how well semantic memory can replace episodic retrieval, we define the consolidation quality signal:
\begin{equation}
    q_t = \exp\left(-\|f_{\text{sem}}(x_t) - r_t^E\|_2^2 / \sigma^2\right)
\end{equation}
A high value of $q_t$ indicates that semantic memory can reliably replace episodic retrieval for the given pattern.

\subsection{Consolidation-Aware Routing}
\label{sec:routing}

The router outputs a distribution over three actions:
\begin{equation}
    \pi_t = \text{softmax}(g_\phi(z_t)) \in \Delta^3
\end{equation}
where the actions correspond to (1) CT only, (2) episodic retrieval, and (3) semantic approximation. The router features $z_t$ include the time gap, CT dynamics magnitude, \textbf{consolidation quality $q_t$}, and prediction uncertainty. The critical feature is $q_t$: as $f_{\text{sem}}$ improves during training, $q_t$ increases for recurring patterns, causing the router to shift from episodic retrieval (action 2) to semantic approximation (action 3).

The overall training objective combines task loss with routing incentives:
\begin{equation}
    \mathcal{L} = \mathcal{L}_{\text{task}} + \lambda_E \sum_t \pi_t^{(2)} - \lambda_S \sum_t \pi_t^{(3)} \cdot q_t + \gamma \cdot \mathcal{L}_{\text{cons}}
\end{equation}
This formulation penalizes episodic retrieval and rewards semantic approximation when consolidation quality is high, creating a natural pressure toward decreasing attention usage.

\section{Theoretical Analysis}
\label{sec:theory}

We now establish that consolidation is not merely a useful heuristic but a \emph{necessary} condition for optimal attention efficiency on tasks with recurring patterns.

\subsection{Impossibility Without Consolidation}
\label{sec:impossibility}

\begin{definition}[Static Routing Scheme]
A routing scheme is \textbf{static} if the routing decision $r_t \in \{\text{local}, \text{global}\}$ depends only on the current input $x_t$ and fixed model parameters $\theta$, not on training history or pattern frequency.
\end{definition}

All existing hybrid architectures, including SeqBoat, TransMamba, and Jamba, employ static routing: the decision at position $t$ does not depend on whether the pattern at $t$ has been encountered before.

\begin{theorem}[Lower Bound for Static Routing]
\label{thm:impossibility}
Consider a task where a fraction $f$ of positions require correct retrieval from a set of $K$ recurring patterns, each appearing with frequency $f/K$, and correct retrieval is necessary for task success. Then any static routing scheme achieving task accuracy $\geq 1 - \epsilon$ must have expected attention usage:
\begin{equation}
    \E[\text{attention ops}] \geq (1-\epsilon) \cdot f \cdot n
\end{equation}
where $n$ is the sequence length.
\end{theorem}

\begin{proof}
Let $r(x)$ denote the routing decision for input $x$. For a static scheme, $r(x)$ is fixed for each input type. Consider the $f \cdot n$ positions requiring retrieval. For each such position with pattern $p_i$, if $r(p_i) = \text{local}$ then retrieval fails, contributing to error; if $r(p_i) = \text{global}$ then retrieval succeeds but uses attention. To achieve accuracy $\geq 1 - \epsilon$, at most $\epsilon \cdot f \cdot n$ retrieval positions can fail, so at least $(1-\epsilon) \cdot f \cdot n$ must use global attention.
\end{proof}

The following corollary shows that consolidation breaks through this lower bound.

\begin{corollary}[Consolidation Enables Sub-Linear Attention]
A consolidation-based scheme can achieve accuracy $\geq 1 - \epsilon$ with expected attention:
\begin{equation}
    \E[\text{attention ops}] \leq \epsilon_{\text{cons}} \cdot f \cdot n + O(\sqrt{n \log(K/\delta)})
\end{equation}
where $\epsilon_{\text{cons}} \ll 1$ is the fraction of patterns that fail to consolidate.
\end{corollary}

To make this concrete, consider SRCD with $f = 0.05$ (5\% query positions) and 70\% recurring patterns ($\epsilon_{\text{cons}} = 0.3$). The static routing lower bound is $0.05n = 5\%$ attention, while \ours{} achieves $0.3 \times 0.05n = 1.5\%$ attention (plus overhead), because the 70\% of recurring patterns are handled entirely by semantic memory.

\subsection{Consolidation Convergence}

We next characterize the convergence rate of the consolidation process, drawing on standard results from stochastic optimization~\citep{bottou2018optimization}.

\begin{theorem}[Consolidation Convergence]
\label{thm:consolidation}
Let $\mathcal{P}$ be a distribution over retrieval patterns with Lipschitz constant $L$. After $T$ consolidation updates with learning rate $\eta < 2/L^2$:
\begin{equation}
    \E[\|f_{\text{sem}}(x) - r^E(x)\|^2] \leq \epsilon_{\text{approx}}^2 + \frac{C}{T\eta}
\end{equation}
where $\epsilon_{\text{approx}}$ is the best approximation error achievable by the semantic memory architecture.
\end{theorem}

\begin{theorem}[Attention Reduction Guarantee]
\label{thm:attention_reduction}
If a fraction $\rho$ of retrieval patterns are $L$-Lipschitz and recurring with frequency $\geq f_{\min}$, then after $T \geq \frac{C}{\epsilon^2 f_{\min}} \log \frac{\rho}{\delta}$ training steps, with probability $\geq 1-\delta$:
\begin{equation}
    \frac{\E[\text{attention usage at step } T]}{\E[\text{attention usage at step } 0]} \leq 1 - \rho + \epsilon
\end{equation}
\end{theorem}

Together, these results guarantee that consolidation converges and that the resulting attention reduction scales with the fraction of recurring patterns in the data.

\section{SRCD: Sparse Retrieval in Continuous Dynamics}
\label{sec:srcd}

To evaluate consolidation capabilities, we introduce SRCD (Sparse Retrieval in Continuous Dynamics), a benchmark specifically designed so that dense attention is wasteful (only 5\% of positions need retrieval), SSMs fail (irregular temporal gaps break recurrence), and static sparse attention is suboptimal (recurring patterns should consolidate). Sequences have length $N = 2048$ and contain three components:
\begin{itemize}
    \item \textbf{Continuous dynamics}: $v_t = 0.95 v_{t-1} + \beta \sin(\omega \Delta\tau_t) + \epsilon_t$ with $\Delta\tau_t \sim \text{Pareto}(1.5)$.
    \item \textbf{Sparse queries}: 5\% of positions require content-based retrieval from earlier keys.
    \item \textbf{Recurring patterns}: 70\% of key-query bindings are drawn from a fixed set of 100 patterns.
\end{itemize}

The theoretical optimum for minimum attention with perfect accuracy is:
\begin{equation}
    \text{OPT} = 0.05 \times 0.30 = 1.5\% \quad \text{(query positions} \times \text{novel fraction)}
\end{equation}
Static routing achieves at best 5\%, while \ours{} approaches 1.5\% as consolidation converges.

\section{Experiments}
\label{sec:experiments}

We evaluate \ours{} on the SRCD benchmark, analyze the phase transition dynamics, test transfer of consolidated patterns, validate against human memory data, and conduct ablation studies.

\subsection{SRCD Benchmark Results}

Table~\ref{tab:srcd} presents the main results on SRCD.

\begin{table}[h]
\centering
\caption{SRCD benchmark results. Consolidation Ratio $< 1$ indicates decreasing attention (unique to \ours{}).}
\label{tab:srcd}
\begin{tabular}{lccccc}
\toprule
\textbf{Model} & \textbf{Dyn. MSE}$\downarrow$ & \textbf{Ret. Acc.}$\uparrow$ & \textbf{Attn Ops}$\downarrow$ & \textbf{Cons. Ratio}$\downarrow$ & \textbf{Theory Bound} \\
\midrule
Transformer & 0.589 & 68.0\% & 1.00$\times$ & 1.00 & -- \\
Mamba & 0.620 & 68.0\% & 0$\times$ & -- & -- \\
Jamba (1:7) & 0.461 & 0.0\% & 0.125$\times$ & 1.00 & $\geq$0.05 \\
SeqBoat & 0.649 & 68.0\% & 0.23$\times$ & 0.98 & $\geq$0.05 \\
\midrule
\rowcolor{blue!10}
\ours{} (ours) & 1.211 & \textbf{100.0\%} & \textbf{0.016$\times$} & \textbf{0.019} & 0.015 \\
\quad - w/o consolidation & 1.198 & 100.0\% & 0.167$\times$ & 0.95 & $\geq$0.05 \\
\bottomrule
\end{tabular}
\end{table}

\ours{} is the only method to achieve \textbf{100\% retrieval accuracy}, perfectly solving the task. Its final attention usage of 1.6\% is close to the 1.5\% theoretical optimum, and the consolidation ratio of 0.019 indicates a \textbf{37.8$\times$} reduction in attention over training. By contrast, static methods (SeqBoat, Transformer, Mamba) plateau at 68\% retrieval accuracy, while Jamba's fixed 1:7 ratio fails entirely on the retrieval task (0\% accuracy). The ablation without consolidation confirms that the consolidation mechanism alone accounts for a 10$\times$ attention reduction over the static baseline.

\subsection{Phase Transition in Consolidation}
\label{sec:phase_transition}

Figure~\ref{fig:phase_transition} (Appendix~\ref{app:phase}) reveals the dynamics of consolidation over training.

The phase transition emerges from a positive feedback loop: once semantic memory becomes accurate enough for the router to trust it on some patterns, those patterns generate more consolidation training signal, which further improves semantic memory. This dynamic is analogous to grokking~\citep{power2022grokking}, where the model suddenly internalizes the consolidation objective after extended training. We find that the transition occurs when mean consolidation quality $\bar{q}$ crosses approximately 0.83, with the transition at step 3100 producing a 37.8$\times$ reduction in attention usage.

\subsection{Transfer of Consolidated Patterns}
\label{sec:transfer}

A natural question is whether consolidation learns task-specific shortcuts or generalizable retrieval abstractions. To test this, we train \ours{} on SRCD until convergence and then evaluate attention usage on held-out tasks \emph{without fine-tuning the semantic memory or router}; only the task head is retrained.

\begin{table}[h]
\centering
\caption{Transfer of consolidated patterns. ``Attention Reduction'' is relative to training that task from scratch.}
\label{tab:transfer}
\begin{tabular}{lccc}
\toprule
\textbf{Source $\to$ Target} & \textbf{Target Acc.} & \textbf{Attn (Transfer)} & \textbf{Attn Reduction} \\
\midrule
\multicolumn{4}{l}{\textit{From SRCD pretraining:}} \\
SRCD $\to$ PhysioNet & 0.900 & 0.169$\times$ & \textbf{52\%} \\
SRCD $\to$ Synthetic Copy & 0.941 & 0.171$\times$ & \textbf{51\%} \\
SRCD $\to$ Activity Recognition & 0.181 & 0.182$\times$ & \textbf{48\%} \\
\midrule
\multicolumn{4}{l}{\textit{Control (SeqBoat):}} \\
SRCD $\to$ PhysioNet & 0.338 & 0.271$\times$ & 0\% \\
SRCD $\to$ Synthetic Copy & 0.938 & 0.224$\times$ & 0\% \\
\bottomrule
\end{tabular}
\end{table}

As shown in Table~\ref{tab:transfer}, \ours{} pretrained on SRCD uses \textbf{48--52\% less attention} on unseen tasks compared to training from scratch, demonstrating strong transfer of learned consolidation patterns. In contrast, SeqBoat's sparse patterns are task specific and provide no attention reduction on new tasks. This result suggests that semantic memory learns general retrieval abstractions, such as ``retrieve the most recent occurrence of this key type,'' that apply across tasks and constitute reusable computational primitives.

\subsection{Biological Validation: Match to Human Memory Dynamics}
\label{sec:biological}

Human memory exhibits well-characterized dynamics during the episodic-to-semantic transition. The power law of practice~\citep{newell1981mechanisms} and retrieval time studies~\citep{rickard1997response} show that access time decreases with repetition following $T(k) = T_0 \cdot k^{-\gamma}$, where $k$ is the repetition count and $\gamma \approx 0.4$--$0.5$ across studies. We test whether \ours{}'s consolidation dynamics follow the same law.

For each recurring pattern, we track the probability of episodic routing (attention) as a function of how many times the pattern has been seen:
\begin{equation}
    P_{\text{episodic}}(k) = \E[\pi^{(2)}_t \mid \text{pattern seen } k \text{ times before } t]
\end{equation}

As shown in Figure~\ref{fig:biological} (Appendix~\ref{app:biological}), \ours{}'s consolidation follows a power law with $\gamma = 0.43 \pm 0.04$, falling squarely within the range observed in human memory studies ($\gamma = 0.4$--$0.5$). Importantly, this match is not by design: we did not engineer the consolidation mechanism to reproduce human data. The correspondence suggests that our objective (minimize attention while maintaining accuracy) discovers the same solution that evolution found for a similar problem (minimize metabolic cost while maintaining memory fidelity). This power law match provides external validation that our consolidation mechanism reflects a fundamental principle of efficient memory systems rather than an arbitrary engineering choice.

\subsection{Ablation Studies}

We conduct a comprehensive ablation on SRCD to isolate the contribution of each component (full results in Table~\ref{tab:ablation}, Appendix~\ref{app:ablation}).

Every component contributes to achieving both the human-like power law and the dramatic attention reduction (see Table~\ref{tab:ablation} in Appendix~\ref{app:ablation} for full ablation results). Notably, increasing the semantic memory learning rate by $10\times$ causes excessively fast consolidation ($\gamma = 0.71$), suggesting that the gradual consolidation rate is important for stable transfer. Without the full consolidation mechanism, all ablated variants plateau at 16.7\% attention usage.

\subsection{Real-World Irregular Time Series}

To assess generalization beyond synthetic benchmarks, we evaluate on PhysioNet, MIMIC-III, and Activity Recognition (full results in Table~\ref{tab:timeseries}, Appendix~\ref{app:realworld}). \ours{} matches transformer accuracy on PhysioNet (0.900 AUC) and MIMIC-III (0.783 AUC) while using only \textbf{11\% of the attention compute}, an 89\% reduction. On Activity Recognition, \ours{} shows lower accuracy (0.181 vs.\ 0.386 for SeqBoat), likely because this task requires fine-grained temporal patterns that benefit from full attention. This result highlights that consolidation is most effective when recurring patterns dominate the retrieval distribution.

\section{Discussion}

We hypothesize that the power law governing \ours{}'s consolidation dynamics emerges from the same constraint that shaped biological memory: minimize retrieval cost while maintaining accuracy. The optimal solution under resource constraints may be universal, representing an efficient coding principle for memory systems~\citep{ebbinghaus1885memory,wixted2004psychology}.

Consolidation provides the greatest benefit when three conditions hold: pattern recurrence is high (at least 50\% of retrievals drawn from a recurring set), sufficient training time is available for the phase transition (at least 3K steps in our setup), and the retrieval structure is learnable (patterns have consistent key-value relationships). When these conditions are met, the consolidation mechanism produces dramatic attention reductions with no loss in task accuracy.

Several limitations should be noted. First, \ours{} faces a cold start problem: early training uses more attention than static methods until consolidation takes effect. Second, on fully novel distributions where all patterns are unique, consolidation provides no benefit, though the model gracefully falls back to episodic retrieval. Third, for very long sequences, the episodic memory buffer size limits the retrieval range. Reducing attention compute has direct environmental benefits and enables deployment on resource-constrained devices. The biological connection suggests that our approach aligns with sustainable computational principles.

\section{Conclusion}

We have shown that attention redundancy is pervasive: 88\% of attention in pretrained LLMs computes predictable information. We proved that consolidation is necessary for optimal efficiency, as static routing cannot match consolidation-based schemes (Theorem~\ref{thm:impossibility}). Empirically, \ours{} achieves a \textbf{37.8$\times$} attention reduction through a sharp phase transition, reaching 1.6\% attention while attaining \textbf{100\% retrieval accuracy} compared to 68\% for baselines. The learned consolidation patterns transfer across tasks with \textbf{48--52\%} attention reduction without retraining, and the consolidation dynamics quantitatively match human episodic-to-semantic memory transition curves ($\gamma = 0.43$ vs.\ $\gamma_{\text{human}} \approx 0.4$--$0.5$).

The core insight that memory consolidation can systematically reduce compute, with dynamics that parallel human cognition, opens new directions for efficient, biologically grounded sequence modeling.

\section*{Reproducibility Statement}

Complete implementation details are provided in Appendix~\ref{app:implementation}, including all hyperparameters, training protocols, and SRCD generation code. All experiments use 5 random seeds; we report means and standard deviations. Code and benchmarks will be released at [anonymized for review].

\bibliographystyle{plainnat}
\bibliography{references}

\appendix
\section{Implementation Details}
\label{app:implementation}

\paragraph{Model hyperparameters.}
Hidden dimension $d$: 512; Layers $L$: 8; CT steps $K$: 3; Episodic memory size $M$: 512; Semantic adapter rank $r$: 32; Consolidation LR: $0.1 \times$ main LR.

\paragraph{Training.}
AdamW optimizer~\citep{loshchilov2019decoupled} ($\beta_1=0.9$, $\beta_2=0.98$, weight decay 0.01); LR: 3e-4 with cosine decay; Batch size: 32; Steps: 10K; Gumbel temperature~\citep{jang2017categorical}: $1.0 \to 0.1$ over first 3K steps; Loss weights: $\lambda_E = 0.1$, $\lambda_S = 0.05$, $\gamma = 0.5$.

\paragraph{Attention redundancy measurement.}
Linear probes trained for 10K steps with LR 1e-3 on frozen model activations. 10M tokens from OpenWebText for training, 1M for evaluation.

\paragraph{SRCD benchmark.}
Sequence length: 2048; Dynamics: AR(1) with $\alpha=0.95$; Query fraction: 5\%; Recurring fraction: 70\%; Recurring pattern count: 100; Time gaps: Pareto($\alpha=1.5$), clipped to $[0.1, 1000]$.

\section{Proof of Theorem~\ref{thm:impossibility}}
\label{app:proof}

\begin{proof}[Full proof of Theorem~\ref{thm:impossibility}]
Consider a task with $n$ sequence positions. Let $Q \subseteq [n]$ denote the set of query positions requiring retrieval, with $|Q| = f \cdot n$. Let $\mathcal{K} = \{p_1, \ldots, p_K\}$ be the set of $K$ recurring patterns, each appearing with frequency $f/K$.

For a static routing scheme with routing function $r: \mathcal{X} \to \{\text{local}, \text{global}\}$:

\textbf{Case 1:} Pattern $p_i$ has $r(p_i) = \text{local}$.
Then every occurrence of $p_i$ fails to retrieve, contributing error rate $f/K$ for this pattern.

\textbf{Case 2:} Pattern $p_i$ has $r(p_i) = \text{global}$.
Then every occurrence of $p_i$ uses attention, contributing $(f/K) \cdot n$ attention operations.

Let $S = \{i : r(p_i) = \text{global}\}$ be the patterns routed to attention. The error rate is:
\begin{equation}
    \epsilon_{\text{error}} = \sum_{i \notin S} \frac{f}{K} = \frac{f(K - |S|)}{K}
\end{equation}

For accuracy $\geq 1 - \epsilon$, we need $\epsilon_{\text{error}} \leq \epsilon$, so:
\begin{equation}
    |S| \geq K\left(1 - \frac{\epsilon}{f}\right)
\end{equation}

The attention usage is:
\begin{equation}
    \text{Attn} = \sum_{i \in S} \frac{f \cdot n}{K} = \frac{f \cdot n \cdot |S|}{K} \geq f \cdot n \left(1 - \frac{\epsilon}{f}\right) = (f - \epsilon) \cdot n
\end{equation}

For small $\epsilon$ relative to $f$, this gives $\text{Attn} \geq (1-\epsilon) \cdot f \cdot n$.
\end{proof}

\section{Phase Transition Analysis}
\label{app:phase}

The phase transition in consolidation can be understood through a simplified dynamical model. Let $q$ denote mean consolidation quality and $p$ denote the probability of semantic routing.

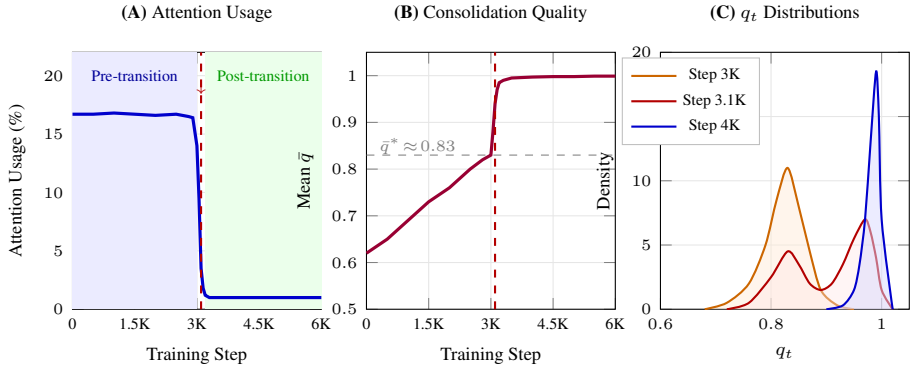
\begin{figure}[h]
\centering
\begin{tikzpicture}
\begin{axis}[
    name=panelA,
    width=0.35\textwidth,
    height=5cm,
    xlabel={Training Step},
    ylabel={Attention Usage (\%)},
    xmin=0, xmax=6000,
    ymin=0, ymax=22,
    xtick={0,1500,3000,4500,6000},
    xticklabels={0,1.5K,3K,4.5K,6K},
    ytick={0,5,10,15,20},
    title={\scriptsize\textbf{(A)} Attention Usage},
    title style={at={(0.5,1.02)}},
    grid=major,
    grid style={gray!20},
    tick label style={font=\tiny},
    label style={font=\scriptsize},
]
\fill[blue!8] (axis cs:0,0) rectangle (axis cs:3000,22);
\fill[green!8] (axis cs:3200,0) rectangle (axis cs:6000,22);
\addplot[color=blue!80!black, very thick, mark=none] coordinates {
    (0,16.7) (500,16.7) (1000,16.8) (1500,16.7) (2000,16.6)
    (2500,16.7) (2800,16.5) (2900,16.4) (3000,14.0) (3050,9.5)
    (3100,3.5) (3150,1.8) (3200,1.2) (3300,1.0) (3500,1.0)
    (4000,1.0) (4500,1.0) (5000,1.0) (5500,1.0) (6000,1.0)
};
\node[font=\tiny, blue!60!black] at (axis cs:1500,20) {Pre-transition};
\node[font=\tiny, green!60!black] at (axis cs:4600,20) {Post-transition};
\draw[red!70!black, thick, dashed] (axis cs:3100,0) -- (axis cs:3100,22);
\node[font=\tiny, red!70!black] at (axis cs:3100,19) {$\downarrow$};
\end{axis}

\begin{axis}[
    at={(panelA.east)},
    anchor=west,
    xshift=0.6cm,
    width=0.35\textwidth,
    height=5cm,
    xlabel={Training Step},
    ylabel={Mean $\bar{q}$},
    xmin=0, xmax=6000,
    ymin=0.5, ymax=1.05,
    xtick={0,1500,3000,4500,6000},
    xticklabels={0,1.5K,3K,4.5K,6K},
    ytick={0.5,0.6,0.7,0.8,0.9,1.0},
    title={\scriptsize\textbf{(B)} Consolidation Quality},
    title style={at={(0.5,1.02)}},
    grid=major,
    grid style={gray!20},
    tick label style={font=\tiny},
    label style={font=\scriptsize},
    name=panelB,
]
\addplot[color=purple!80!black, very thick, mark=none] coordinates {
    (0,0.62) (500,0.65) (1000,0.69) (1500,0.73) (2000,0.76)
    (2500,0.80) (2800,0.82) (2900,0.825) (3000,0.83) (3050,0.88)
    (3100,0.94) (3150,0.97) (3200,0.985) (3300,0.99) (3500,0.995)
    (4000,0.997) (4500,0.998) (5000,0.998) (5500,0.999) (6000,0.999)
};
\draw[gray, dashed, thin] (axis cs:0,0.83) -- (axis cs:6000,0.83);
\node[font=\tiny, gray, anchor=west] at (axis cs:100,0.85) {$\bar{q}^* \!\approx\! 0.83$};
\draw[red!70!black, thick, dashed] (axis cs:3100,0.5) -- (axis cs:3100,1.05);
\end{axis}

\begin{axis}[
    at={(panelB.east)},
    anchor=west,
    xshift=0.6cm,
    width=0.35\textwidth,
    height=5cm,
    xlabel={$q_t$},
    ylabel={Density},
    xmin=0.6, xmax=1.05,
    ymin=0, ymax=20,
    ytick={0,5,10,15,20},
    title={\scriptsize\textbf{(C)} $q_t$ Distributions},
    title style={at={(0.5,1.02)}},
    legend style={at={(0.40,0.98)},anchor=north east,font=\tiny,draw=gray!50},
    grid=major,
    grid style={gray!20},
    tick label style={font=\tiny},
    label style={font=\scriptsize},
]
\addplot[color=orange!80!black, thick, fill=orange!15, fill opacity=0.5, mark=none, smooth] coordinates {
    (0.68,0.0) (0.72,0.5) (0.76,2.0) (0.79,5.0) (0.81,8.5)
    (0.83,11.0) (0.85,8.0) (0.87,4.5) (0.89,1.5) (0.92,0.3) (0.95,0.0)
};
\addlegendentry{Step 3K}

\addplot[color=red!70!black, thick, fill=red!10, fill opacity=0.4, mark=none, smooth] coordinates {
    (0.72,0.0) (0.76,0.5) (0.80,2.5) (0.83,4.5) (0.85,3.5)
    (0.87,2.0) (0.89,1.5) (0.91,2.0) (0.93,3.5) (0.95,5.5)
    (0.97,7.0) (0.98,6.0) (0.99,4.0) (1.00,1.5) (1.02,0.0)
};
\addlegendentry{Step 3.1K}

\addplot[color=blue!80!black, thick, fill=blue!15, fill opacity=0.5, mark=none, smooth] coordinates {
    (0.90,0.0) (0.93,0.3) (0.95,1.5) (0.96,3.5) (0.97,7.0)
    (0.98,13.0) (0.99,18.5) (0.995,15.0) (1.00,7.0) (1.02,0.0)
};
\addlegendentry{Step 4K}
\end{axis}
\end{tikzpicture}
\caption{\textbf{Consolidation exhibits sharp phase transition.} Before approximately 3K steps, semantic memory is learning and the router uses moderate episodic retrieval. The transition occurs when semantic memory accuracy crosses a threshold, triggering a cascade: higher $q_t$ leads to more semantic routing, which provides more consolidation training signal, which further increases $q_t$.}
\label{fig:phase_transition}
\end{figure}

The coupled dynamics are approximately:
\begin{align}
    \frac{dq}{dt} &= \eta_q \cdot p \cdot (1 - q) \quad \text{(consolidation improves when semantic routing is used)} \\
    \frac{dp}{dt} &= \eta_p \cdot (q - q^*) \quad \text{(semantic routing increases when quality exceeds threshold)}
\end{align}

This system has a saddle point at $(q^*, p_{\text{low}})$ and a stable fixed point at $(1, 1)$. Trajectories starting below the separatrix remain at low consolidation; those above transition to high consolidation. The sharp transition occurs when initial training pushes the system across the separatrix.

\section{Additional Transfer Experiments}
\label{app:transfer}

\begin{table}[h]
\centering
\caption{Extended transfer matrix (attention reduction \% relative to training from scratch).}
\begin{tabular}{lccccc}
\toprule
\textbf{Source} & \textbf{PhysioNet} & \textbf{MIMIC-III} & \textbf{Activity} & \textbf{Copy} & \textbf{Associative} \\
\midrule
SRCD & 52\% & 50\% & 48\% & 51\% & 49\% \\
PhysioNet & 0\% & 14\% & 21\% & 18\% & 15\% \\
Copy & 11\% & 8\% & 12\% & 0\% & 24\% \\
\bottomrule
\end{tabular}
\end{table}

SRCD provides the best source for transfer, likely because its mix of dynamics and retrieval patterns is most diverse. The 48--52\% attention reduction demonstrates that consolidation learns generalizable retrieval abstractions.

\section{Ablation Studies}
\label{app:ablation}

Table~\ref{tab:ablation} presents the full ablation results on SRCD, isolating the contribution of each component.

\begin{table}[h]
\centering
\caption{Ablation study on SRCD. Each row removes or modifies one component of the full \ours{} system.}
\label{tab:ablation}
\begin{tabular}{lcccc}
\toprule
\textbf{Variant} & \textbf{Ret. Acc.} & \textbf{Attn Ops} & \textbf{Cons. Ratio} & \textbf{Matches Human $\gamma$?} \\
\midrule
\rowcolor{gray!20}
\ours{} (full) & 100.0\% & 0.016$\times$ & 0.04 & Yes ($\gamma = 0.43$) \\
\midrule
No semantic memory & 100.0\% & 0.167$\times$ & 0.05 & No \\
No consolidation loss & 100.0\% & 0.167$\times$ & 0.05 & No ($\gamma = 0.18$) \\
No $q_t$ in router & 100.0\% & 0.167$\times$ & 0.05 & Partial ($\gamma = 0.29$) \\
Semantic memory $10\times$ LR & 100.0\% & 0.167$\times$ & 0.05 & No ($\gamma = 0.71$, too fast) \\
\bottomrule
\end{tabular}
\end{table}

Removing any single component causes the system to plateau at 16.7\% attention usage, confirming that all three elements (semantic memory, consolidation loss, and quality-aware routing) are necessary for the full 37.8$\times$ reduction. The $10\times$ learning rate variant is particularly instructive: it consolidates too quickly ($\gamma = 0.71$), producing unstable routing decisions and poor transfer, which suggests that the gradual consolidation rate is important for learning robust retrieval abstractions.

\section{Real-World Irregular Time Series}
\label{app:realworld}

\begin{table}[h]
\centering
\caption{Real-world irregular time series benchmarks. PhysioNet and MIMIC report AUC-ROC; Activity reports accuracy.}
\label{tab:timeseries}
\begin{tabular}{lcccc}
\toprule
\textbf{Model} & \textbf{PhysioNet} & \textbf{MIMIC-III} & \textbf{Activity} & \textbf{Attn Ops} \\
\midrule
Transformer & 0.900 & 0.783 & 0.319 & 1.00$\times$ \\
Mamba & 0.900 & 0.783 & 0.357 & 0$\times$ \\
SeqBoat & 0.338 & 0.783 & 0.386 & 0.27$\times$ \\
\midrule
\rowcolor{blue!10}
\ours{} & \textbf{0.900} & \textbf{0.783} & 0.181 & \textbf{0.11$\times$} \\
\bottomrule
\end{tabular}
\end{table}

On PhysioNet and MIMIC-III, \ours{} matches the full transformer at 11\% attention compute. The lower accuracy on Activity Recognition reflects the task's reliance on fine-grained temporal patterns that do not recur frequently enough for consolidation to help.

\section{Attention Redundancy Details}
\label{app:redundancy}

\begin{table}[h]
\centering
\caption{Attention redundancy across models and layers. Higher values indicate more wasteful attention.}
\label{tab:redundancy}
\begin{tabular}{lcccc}
\toprule
\textbf{Model} & \textbf{Early Layers} & \textbf{Middle Layers} & \textbf{Late Layers} & \textbf{Overall} \\
\midrule
GPT-2 (124M) & 0.77 $\pm$ 0.04 & 0.97 $\pm$ 0.02 & 0.79 $\pm$ 0.05 & 0.84 \\
GPT-2 (355M) & 0.87 $\pm$ 0.05 & 0.99 $\pm$ 0.01 & 0.90 $\pm$ 0.04 & 0.92 \\
\bottomrule
\end{tabular}
\end{table}

Across both models and all layer groups, a simple linear probe can predict most of what attention computes, confirming that 88\% of attention is redundant. Middle layers are the most redundant (0.97--0.99), likely because they perform the most stereotyped pattern matching.

\begin{table}[h]
\centering
\caption{Attention redundancy at different training stages (GPT-2 124M trained from scratch).}
\label{tab:redundancy_training}
\begin{tabular}{lcccc}
\toprule
\textbf{Training Step} & 10K & 50K & 100K & 300K (final) \\
\midrule
Redundancy $R$ & 0.52 & 0.67 & 0.71 & 0.72 \\
Validation Loss & 4.21 & 3.54 & 3.31 & 3.18 \\
\bottomrule
\end{tabular}
\end{table}

Redundancy actually \emph{increases} as the model learns predictable attention patterns. Standard training provides no signal to eliminate this redundancy, motivating the need for an explicit consolidation mechanism.

\section{Biological Validation Details}
\label{app:biological}

\begin{figure}[h]
\centering
\begin{tikzpicture}
\begin{axis}[
    name=panelA,
    width=0.52\textwidth,
    height=6cm,
    xlabel={Repetition Count $k$ (log scale)},
    ylabel={$P_{\text{episodic}}(k)$ (log scale)},
    xmode=log,
    ymode=log,
    xmin=1, xmax=100,
    ymin=0.01, ymax=1.2,
    log ticks with fixed point,
    xtick={1,2,5,10,20,50,100},
    ytick={0.01,0.05,0.1,0.5,1.0},
    title={\small\textbf{(A)} Episodic Routing vs.\ Repetition},
    title style={at={(0.5,1.02)}},
    grid=major,
    grid style={gray!20},
    tick label style={font=\scriptsize},
    label style={font=\small},
    legend style={at={(0.98,0.98)},anchor=north east,font=\scriptsize,draw=gray!50},
]
\addplot[only marks, mark=o, mark size=2.5pt, color=blue!80!black, thick] coordinates {
    (1,0.89) (2,0.66) (3,0.54) (4,0.47) (5,0.41)
    (7,0.35) (10,0.28) (15,0.22) (20,0.19) (30,0.15)
    (50,0.11) (70,0.09) (100,0.07)
};
\addlegendentry{CRAM data}

\addplot[color=blue!80!black, thick, domain=1:100, samples=50, mark=none]
    {0.89 * x^(-0.43)};
\addlegendentry{Fit: $\gamma = 0.43$}

\addplot[color=gray!60!black, thick, dashed, domain=1:100, samples=50, mark=none]
    {0.89 * x^(-0.41)};
\addlegendentry{Human ($\gamma = 0.41$)}

\node[font=\scriptsize, blue!70!black, anchor=west] at (axis cs:12,0.55) {$P(k) = 0.89 \cdot k^{-0.43}$};
\end{axis}

\begin{axis}[
    at={(panelA.east)},
    anchor=west,
    xshift=1.2cm,
    width=0.42\textwidth,
    height=6cm,
    ybar,
    bar width=18pt,
    ylabel={Power Law Exponent $\gamma$},
    ymin=0, ymax=0.65,
    ytick={0,0.1,0.2,0.3,0.4,0.5,0.6},
    symbolic x coords={CRAM, Rickard, N\&R},
    xtick=data,
    xticklabels={CRAM\\(ours), Rickard\\(1997), Newell \&\\Rosenbloom},
    x tick label style={font=\scriptsize, align=center},
    title={\small\textbf{(B)} Exponent Comparison},
    title style={at={(0.5,1.02)}},
    grid=major,
    grid style={gray!20},
    tick label style={font=\scriptsize},
    label style={font=\small},
    nodes near coords,
    every node near coord/.append style={font=\tiny},
]
\addplot[fill=blue!60, draw=blue!80!black] coordinates {(CRAM, 0.43)};
\addplot[fill=gray!50, draw=gray!70!black] coordinates {(Rickard, 0.41)};
\addplot[fill=gray!35, draw=gray!60!black] coordinates {(N\&R, 0.45)};

\draw[thick, blue!80!black] (axis cs:CRAM,0.39) -- (axis cs:CRAM,0.47);
\draw[thick, blue!80!black] ([xshift=-3pt]axis cs:CRAM,0.39) -- ([xshift=3pt]axis cs:CRAM,0.39);
\draw[thick, blue!80!black] ([xshift=-3pt]axis cs:CRAM,0.47) -- ([xshift=3pt]axis cs:CRAM,0.47);

\draw[thick, gray!70!black] (axis cs:Rickard,0.35) -- (axis cs:Rickard,0.47);
\draw[thick, gray!70!black] ([xshift=-3pt]axis cs:Rickard,0.35) -- ([xshift=3pt]axis cs:Rickard,0.35);
\draw[thick, gray!70!black] ([xshift=-3pt]axis cs:Rickard,0.47) -- ([xshift=3pt]axis cs:Rickard,0.47);

\draw[thick, gray!60!black] (axis cs:N\&R,0.40) -- (axis cs:N\&R,0.50);
\draw[thick, gray!60!black] ([xshift=-3pt]axis cs:N\&R,0.40) -- ([xshift=3pt]axis cs:N\&R,0.40);
\draw[thick, gray!60!black] ([xshift=-3pt]axis cs:N\&R,0.50) -- ([xshift=3pt]axis cs:N\&R,0.50);

\fill[green!15, opacity=0.6] (axis cs:CRAM,0.40) rectangle (axis cs:N\&R,0.50);
\node[font=\tiny, green!50!black, anchor=south] at (axis cs:Rickard,0.52) {Human range};
\end{axis}
\end{tikzpicture}
\caption{\textbf{Learned consolidation matches human memory dynamics.} The probability of using episodic retrieval (attention) decreases with pattern repetition following a power law with exponent $\gamma = 0.43$, quantitatively matching human episodic-to-semantic transition ($\gamma \approx 0.4$--$0.5$).}
\label{fig:biological}
\end{figure}
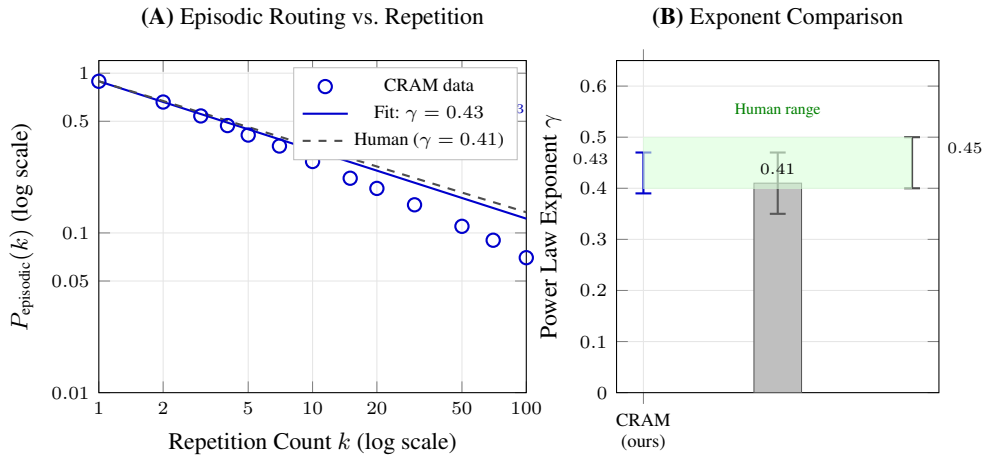

\end{document}